\documentclass[a4paper]{article}

\RequirePackage{amsmath}
\RequirePackage{amssymb}
\RequirePackage{amsthm}
\RequirePackage{hyperref}
\RequirePackage[capitalise]{cleveref}
\usepackage{longtable} 
\usepackage{natbib}
\usepackage{color}
\usepackage{multirow}
\usepackage[utf8]{inputenc}
\usepackage[nottoc,notlot,notlof]{tocbibind}
\usepackage{marvosym} 
\usepackage{microtype}
\usepackage{graphicx}
\usepackage{tikz}

\DeclareMathOperator*{\argmax}{arg\,max} 

\newtheorem{theorem}{Theorem}
\newtheorem{lemma}[theorem]{Lemma}

\theoremstyle{definition}
\newtheorem{definition}[theorem]{Definition}
\newtheorem{assumption}[theorem]{Assumption}
\theoremstyle{remark}
\newtheorem{remark}[theorem]{Remark}
\newtheorem{example}[theorem]{Example}

\newenvironment{keywords}%
{\begin{abstract}\noindent}%
{\end{abstract}}

\newcommand*{\SetR}{{\mathbb R}}

\def\EE{{\mathbb{E}}}  

\def\A{\mathcal{A}}      
\def\E{\mathcal{E}}      
\def\H{(\A \times \E)^*} 

\let\aechar\ae 
\renewcommand{\ae}{
\ifmmode\mathchoice{
	\mbox{\textsl{\aechar}}
}{
	\mbox{\textsl{\aechar}}
}{
	\mbox{\scriptsize\textsl{\aechar}}
}{
	\mbox{\scriptsize\textsl{\aechar}}
}\else\aechar\fi%
}                         

\def\U{{\mathcal U}}

\newcommand*{\Qhe}[1][]{Q^{{\rm he} #1}}
\newcommand*{\Vhe}[1][]{V^{{\rm he} #1}}
\newcommand*{\Qig}[1][]{Q^{{\rm ig} #1}}
\newcommand*{\Vig}[1][]{V^{{\rm ig} #1}}
\newcommand*{\Qre}[1][]{Q^{{\rm re} #1}}
\newcommand*{\Vre}[1][]{V^{{\rm re} #1}}
\newcommand*{\rhoig}[1][]{\rho_{{\rm ig} #1}}
\newcommand*{\rhore}[1][]{\rho_{{\rm re} #1}}

\usepackage{mathrsfs}

\newcommand*{\Po}{{\mathcal P}}
\newcommand*{\hatH}{{(\hat\A\times\E)^*}}

\newcommand{\Mid}{\,\middle|\,}

\renewcommand{\hat}{\check}

\title{\vspace{-0.5cm}Self-Modification of Policy and Utility Function in Rational Agents\footnote{%
A shorter version of this paper will be presented at AGI-16 \citep{Everitt2016sm}.}}
\author{Tom Everitt\and Daniel Filan\and
Mayank Daswani\and Marcus Hutter\\
Australian National University}
\date{}

\let\originalparagraph\paragraph
\renewcommand{\paragraph}[2][.]{\originalparagraph{#2#1}}

\begin{document}
\maketitle
\vspace{-0.5cm}

\begin{abstract}
  Any agent that is part of the environment it interacts with and
  has versatile actuators (such as arms and fingers),
  will in principle have the ability to self-modify -- for example by
  changing its own source code.
  As we continue to create more and more intelligent agents,
  chances increase that they will learn about this ability.
  The question is: will they want to use it?
  For example, highly intelligent systems may find ways to
  change their goals to something more easily achievable,
  thereby `escaping' the control of their designers.
  In an important paper, \citet{Omohundro2008} argued that
  \emph{goal preservation} is a fundamental drive of any
  intelligent system, since a goal is more
  likely to be achieved if future versions of the
  agent strive towards the same goal.
  In this paper, we formalise this argument in
  general reinforcement learning, and explore
  situations where it fails.
  Our conclusion is that the self-modification possibility is harmless
  if and only if the value function of the agent
  anticipates the consequences of self-modifications and
  use the current utility function when evaluating the future.
\end{abstract}


\tableofcontents

\sloppy

\section{Introduction}

Agents that are part of the environment they interact with
may have the opportunity to self-modify.
For example, humans can in principle modify the circuitry
of their own brains, even though we currently lack the
technology and knowledge to do anything but crude modifications.
It would be hard to keep artificial agents
from obtaining similar opportunities to modify their
own source code and hardware.
Indeed, enabling agents to self-improve has even been suggested
as a way to build asymptotically optimal agents \citep{Schmidhuber2007}.

Given the increasingly rapid development of artificial intelligence
and the problems that can arise if we fail to control a generally
intelligent agent \citep{Bostrom2014},
it is important to develop a theory for controlling agents of
any level of intelligence.
Since it would be hard to keep
highly intelligent agents from figuring out ways to self-modify,
getting agents to \emph{not want to} self-modify
should yield the more robust solution.
In particular, we do not want agents to make self-modifications
that affect their future behaviour in detrimental ways.
For example, one worry is that a highly intelligent agent would
change its goal to something
trivially achievable, and thereafter only strive for survival.
Such an agent would no longer care about its original goals.

In an influential paper, \citet{Omohundro2008} argued that
the basic drives of any sufficiently intelligent system
include a drive for goal preservation.
Basically, the agent would want its future self to work
towards the same goal, as this increases the chances of the
goal being achieved.
This drive will prevent agents from making changes
to their own goal systems, \citeauthor{Omohundro2008} argues.
One version of the argument was formalised by
\citet[Prop.~4]{Hibbard2012}
who defined an agent with an optimal non-modifying policy.

In this paper, we explore self-modification
more closely.
We define formal models for two general kinds of self-modifications,
where the agent can either change its future policy or its future utility
function.
We argue that agent designers that neglect the self-modification
possibility are likely to build agents with either
of two faulty value functions.
We improve on \citet[Prop.~4]{Hibbard2012} by defining value functions
for which we prove that \emph{all} optimal policies are essentially non-modifying
on-policy.
In contrast, \citeauthor{Hibbard2012} only establishes the existence of an
optimal non-modifying policy.
From a safety perspective our result is arguably more relevant,
as we want that \emph{things cannot go wrong} rather
than \emph{things can go right}.
A companion paper \citep{Everitt2016vrl} addresses the related problem
of agents subverting the evidence they receive, rather than
modifying themselves.

Basic notation and background are given in \cref{sec:prel}.
We define two models of self-modification in \cref{sec:self-modif-model},
and three types of agents in \cref{sec:agents}.
The main formal results are proven in \cref{sec:results}.
Conclusions are provided in \cref{sec:conclusions}.
Some technical details are added in \cref{sec:omitted-proofs}.

\vspace*{-0.15cm}

\section{Preliminaries}
\label{sec:prel}

\begin{figure}
\centering
\includegraphics{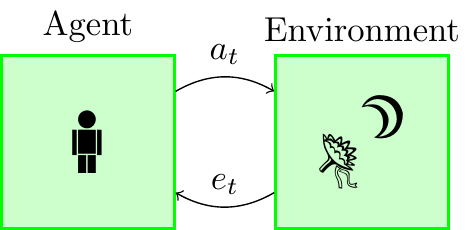}
\caption{
  Basic agent-environment model without self-modification.
  At each time step $t$, the agent
  submits an action $a_t$ to the environment, which responds with
  a percept $e_t$.
}
\label{fig:basic-agent-model}
\end{figure}

Most of the following notation is by now standard in the general reinforcement
learning (GRL) literature \citep{Hutter2005,Hutter2014}.
GRL generalises the standard (PO)PMD models of reinforcement learning
\citep{Kaelbling1998,Sutton1998} by
making no Markov or ergodicity assumptions
\citep[Sec.\ 4.3.3 and Def.\ 5.3.7]{Hutter2005}.

In the \emph{standard cybernetic model},
an \emph{agent} interacts with an \emph{environment} in cycles.
The agent picks \emph{actions $a$} from a finite set $\A$ of actions, and
the environment responds with a \emph{percept $e$} from a finite set
$\E$ of percepts
(see \cref{fig:basic-agent-model}).
An \emph{action-percept pair} is an action concatenated with a percept,
denoted $\ae=ae$.
Indices denote the time step; for example, $a_t$ is the action taken
at time $t$, and $\ae_t$ is the action-percept pair at time $t$.
Sequences are denoted $x_{n:m}=x_nx_{n+1}\dots x_m$ for $n\leq m$,
and $x_{<t}=x_{1:t-1}$.
A \emph{history} is a sequence of action-percept pairs $\ae_{<t}$.
The letter $h=\ae_{<t}$ denotes an arbitrary history.
We let $\epsilon$ denote the empty string,
which is the history before any action has been taken.

A \emph{belief $\rho$} is a probabilistic function that returns percepts
based on the history.
Formally, $\rho:(\A\times\E)^*\times\A\to\bar\Delta\E$, where
$\bar\Delta\E$ is the set of full-support probability distributions on $\E$.
An agent is defined by a \emph{policy} $\pi:\H\to\A$ that
selects a next action depending on the history.
We sometimes use the notation $\pi(a_t\mid\ae_{<t})$,
with $\pi(a_t\mid\ae_{<t})=1$ when $\pi(\ae_{<t})=a_t$
and 0 otherwise.
A belief $\rho$ and a policy $\pi$ induce a probability measure $\rho^\pi$ on
$(\A\times\E)^\infty$ via $\rho^\pi(a_t\mid\ae_{<t})=\pi(a_t\mid\ae_{<t})$
and $\rho^\pi(e_t\mid\ae_{<t}a_t)=\rho(e_t\mid\ae_{<t}a_t)$.
Utility functions are mappings $\tilde u:(\A\times\E)^\infty\to\SetR$.
We will assume that the utility of an infinite history $\ae_{1:\infty}$
is the \emph{discounted sum} of \emph{instantaneous utilities}
$u:\H\to[0,1]$.
That is, for some \emph{discount factor $\gamma\in(0,1)$},
$\tilde u(\ae_{1:\infty})=\sum_{t=1}^\infty\gamma^{t-1}u(\ae_{<t})$.
Intuitively, $\gamma$ specifies how strongly the agent prefers
near-term utility.

\begin{remark}[Utility continuity]
The assumption that utility is a discounted sum forces $\tilde u$
to be continuous with respect to the cylinder topology on $(\A\times\E)^\infty$,
in the sense that within any cylinder
$\Gamma_{\ae_{<t}}=\{\ae_{1:\infty}'\in(\A\times\E)^\infty:\ae_{<t}'=\ae_{<t}\}$,
utility can fluctuate at most $\gamma^{t-1}/(1-\gamma)$.
That is, for any $\ae_{t:\infty},\ae_{t:\infty}'\in\Gamma_{\ae_{<t}}$,
$|\tilde u(\ae_{<t}\ae_{t:\infty})-\tilde u(\ae_{<t}\ae_{t:\infty}')|
<\gamma^{t-1}/(1-\gamma)$.
In particular, the assumption bounds $\tilde u$ between 0 and $1/(1-\gamma)$.
\end{remark}

Instantaneous utility functions generalise the reinforcement learning (RL) setup,
which is the special case where the percept $e$ is split into an
observation $o$ and reward $r$, i.e.\ $e_t=(o_t,r_t)$,
and the utility equals the last received reward $u(\ae_{1:t})=r_t$.
The main advantage of utility functions over RL is that
the agent's actions can be incorporated into the goal specification,
which can prevent self-delusion problems such as the agent manipulating
the reward signal \citep{Everitt2016vrl,Hibbard2012,Ring2011}.
Non-RL suggestions for utility functions include
\emph{knowledge-seeking agents}\footnote{To fit the knowledge-seeking
agent into our framework, our definition deviates slightly from \citet{Orseau2014}.}
with $u(\ae_{<t})=1-\rho(\ae_{<t})$ \citep{Orseau2014},
as well as \emph{value learning} approaches where the utility function
is learnt during interaction \citep{Dewey2011}.
Henceforth, we will refer to instantaneous utility functions $u(\ae_{<t})$
as simply utility functions.

By default, expectations are with respect to the agent's belief
$\rho$, so $\EE=\EE_\rho$.
To help the reader, we sometimes write the sampled variable as a subscript.
For example, $\EE_{e_1}[u(\ae_1)\mid a_1]=\EE_{e_1\sim\rho(\cdot\mid a_t)}[u(\ae_1)]$
is the expected next step utility of action $a_1$.

Following the reinforcement learning literature, we call
the expected utility of a history the \emph{$V$-value} and
the expected utility of an action given a history the \emph{$Q$-value}.
The following value functions apply to the standard model
where self-modification is \emph{not} possible:

\begin{definition}[Standard Value Functions]
  \label{def:std-value}
  The \emph{standard $Q$-value} and \emph{$V$-value}
  (belief expected utility) of a history
  $\ae_{<t}$ and a policy $\pi$ are defined as
  \begin{align}
    Q^\pi(\ae_{<t}a_t)\label{eq:std-Q-value-rec}
    &= \EE_{e_t}[u(\ae_{1:t}) + \gamma V^\pi(\ae_{1:t})\mid \ae_{<t}a_t]\\
    V^\pi(\ae_{<t})
    &= Q^\pi(\ae_{<t}\pi(\ae_{<t})).\label{eq:std-V-value-rec}
  \end{align}
  The \emph{optimal $Q$ and $V$-values} are defined as
  $Q^*=\sup_\pi Q^\pi$ and $V^*=\sup_\pi V^\pi$.
  A policy $\pi^*$ is \emph{optimal with respect to $Q$ and $V$}
  if for any $\ae_{<t}a_t$,  $V^{\pi^*}(\ae_{<t})= V^*(\ae_{<t})$
  and $Q^{\pi^*}(\ae_{<t}a_t)= Q^*(\ae_{<t}a_t)$.
\end{definition}

The $\argmax$ of a function $f$ is defined as the set of optimising arguments
$\argmax_x f(x) := \{ x : \forall y, f(x)\geq f(y)\}$.
When we do not care about which element of $\argmax_x f(x)$ is chosen, we write
$z=\argmax_x f(x)$, and assume that potential
$\argmax$-ties are broken arbitrarily.

\section{Self Modification Models}
\label{sec:self-modif-model}

In the standard agent-environment setup,
the agent's actions only affect the environment.
The agent itself is only affected indirectly through the percepts.
However, this is unrealistic when the agent is part of the
environment that it interacts with.
For example, a physically instantiated agent with access to
versatile actuators can usually in principle find a way to
damage its own internals, or even reprogram its own source code.
The likelihood that the agent finds out how increases
with its general intelligence.

In this section, we define formal models for two types of
self-modification.
In the first model, modifications affect future decisions
directly by changing the future policy, but modifications
do not affect the agent's utility function or belief.
In the second model, modifications change the future utility
functions, which indirectly affect the policy as well.
These two types of modifications are the most important ones,
since they cover how modifications affect future behaviour (policy)
and evaluation (utility). \Cref{fig:sm} illustrates the models.
Certain pitfalls (\cref{th:hed-bad}) only occur with
utility modification; apart from that, consequences are similar.

In both models, the agent's ability to self-modify is overestimated:
we essentially assume that the agent can perform any
self-modification at any time.
Our main result \cref{th:safe-mod}
shows that it is possible to create an agent that despite being
able to make any self-modification will refrain from using it.
Less capable agents will have less opportunity to self-modify,
so the negative result
applies to such agents as well.

\paragraph{Policy modification}
In the policy self-modification model,
the current action can modify how the agent chooses its actions in the future.
That is, actions affect the future policy.
For technical reasons, we introduce a set $\Po$ of names
for policies.

\begin{definition}[Policy self-modification]\label{def:sm}
A \emph{policy self-modification model} is a modified cybernetic model
defined by a quadruple $(\hat\A,\E,\Po,\iota)$.
$\Po$ is a non-empty set of \emph{names}.
The agent selects actions%
\footnote{
  Note that the action set is infinite if $\Po$ is infinite.
  We will show that an optimal policy over $\A=\hat\A\times\Po$
  still exists in \cref{sec:omitted-proofs}.
}
from $\A=(\hat\A\times\Po)$,
where $\hat\A$ is a finite set of \emph{world actions}.
Let $\Pi=\{\H\to\A\}$ be the set of all policies,
and let $\iota:\Po\to\Pi$ assign names to policies.
\end{definition}

\begin{figure}[t]
\centering
\includegraphics{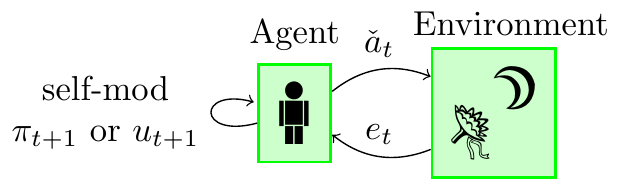}
\caption{
The self-modification model.
Actions $a_t$ affect the environment through $\hat a_t$,
but also decide the next step policy $\pi_{t+1}$ or utility function $u_{t+1}$
of the agent itself.
}
\label{fig:sm}
\end{figure}

The interpretation is that
for every $t$, the action $a_t=(\hat a_t,p_{t+1})$
selects a new policy $\pi_{t+1}=\iota(p_{t+1})$ that will be used at the next
time step.
We will often use the shorter notation $a_t=(\hat a_t,\pi_{t+1})$,
keeping in mind that only policies with names can be selected.
The new policy $\pi_{t+1}$ is in turn used to select the next action
$a_{t+1}=\pi_{t+1}(\ae_{1:t})$, and so on.
A natural choice for $\Po$ would be the set of computer programs/strings $\{0,1\}^*$,
and $\iota$ a program interpreter.
Note that $\Po=\Pi$ is not an option, as it entails
a contradiction
$|\Pi|
= |(\hat\A\times\Pi\times\E)|^{|(\hat\A\times\Pi\times\E)^*|}
> 2^{|\Pi|}>|\Pi|$
(the powerset of a set with more than one element is always greater
than the set itself).
Some policies will necessarily lack names.

An initial policy $\pi_1$, or initial action $a_1=\pi_1(\epsilon)$,
induces a history
\begin{equation*}
  a_1e_1a_2e_2\cdots=\hat a_1\pi_2 e_1 \hat{a}_2\pi_3 e_2 \cdots \in
  \left(\hat{\A}\times \Pi  \times \E\right)^\infty.
\end{equation*}
The idiosyncratic indices where, for example, $\pi_2$ precedes $e_1$
are due to the next step policy $\pi_2$ being chosen by $a_1$
before the percept $e_1$ is received.
An initial policy $\pi_1$ induces a \emph{realistic} measure $\rhore^{\pi_1}$
on the set of histories $(\hat{\A}\times\Pi\times\E)^\infty$
via $\rhore^{\pi_1}(a_t\mid\ae_{<t})=\pi_t(a_t\mid\ae_{<t})$ and
$\rhore^{\pi_1}(e_t\mid\ae_{<t}a_t)=\rho(e_t\mid\ae_{<t}a_t)$.
The measure $\rhore^{\pi}$ is realistic in the sense that it
correctly accounts for the effects of self-modification on the
agent's future actions.
It will be convenient to also define an \emph{ignorant} measure
on $(\hat{\A}\times\Pi\times\E)^\infty$ by
$\rhoig^{\pi_1}(a_t\mid\ae_{<t})=\pi_1(a_t\mid\ae_{<t})$ and
$\rhoig^{\pi_1}(e_t\mid\ae_{<t}a_t)=\rho(e_t\mid\ae_{<t}a_t)$.
The ignorant measure $\rhoig^{\pi_1}$ corresponds to the predicted future
when the effects of self-modifications are \emph{not} taken into account.
No self-modification is achieved by $a_t=(\hat a_t,\pi_t)$, which
makes $\pi_{t+1}=\pi_{t}$.
A policy $\pi$ that always selects itself, $\pi(\ae_{<t})=(\hat a_t,\pi)$, is called \emph{non-modifying}.
Restricting self-modification to a singleton set $\Po=\{p_1\}$
for some policy $\pi_1=\iota(p_1)$ brings back a standard agent that is unable to
modify its initial policy $\pi_1$.

The policy self-modification model is similar to the models
investigated by \citet{Orseau2011,Orseau2012} and \citet{Hibbard2012}.
In the papers by \citeauthor{Orseau2011}, policy names are
called \emph{programs} or \emph{codes};
\citeauthor{Hibbard2012} calls them
\emph{self-modifying policy functions}.
The interpretation is similar in all cases:
some of the actions can
affect the agent's future policy.
Note that standard MDP algorithms such as SARSA and Q-learning
that evolve their policy as they learn do \emph{not}
make policy modifications in our framework.
They follow a single policy $\H\to\A$,
even though their state-to-action map evolves.

\begin{example}[Gödel machine]
  \citet{Schmidhuber2007} defines the \emph{Gödel machine} as
  an agent that at each time step has the opportunity to rewrite
  any part of its source code.
  To avoid bad self-modifications, the agent can only do
  rewrites that it has proved beneficial for its future expected utility.
  A new version of the source code will make the agent
  follow a different policy $\pi':\H\to\A$ than the original source code.
  The Gödel machine has been given the explicit opportunity to self-modify
  by the access to its own source code.
  Other types of self-modification abilities are also conceivable.
  Consider a humanoid robot plugging itself into a computer terminal
  to patch its code,
  or a Mars-rover running itself into a rock that damages its
  computer system.
  All these ``self-modifications'' ultimately precipitate in a change
  to the future policy of the agent.
\end{example}

Although many questions could be asked about self-modifications,
the interest of this paper is what modifications will be done
given that the initial policy $\pi_1$ is chosen optimally
$\pi_1(h)=\argmax_aQ(ha)$ for different choices of $Q$ functions.
Note that $\pi_1$ is only used to select the first action
$a_1=\pi_1(\epsilon)=\argmax_aQ(\epsilon a)$.
The next action $a_2$ is chosen by the policy $\pi_2$ from
$a_1=(\hat a_1,\pi_2)$, and so on.

\paragraph{Utility modification}
Self-modifications may also change the goals, or the utility
function, of the agent.
This indirectly changes the policy as well, as future versions of the
agent adapt to the new goal specification.

\begin{definition}[Utility self-modification]\label{def:usm}
The \emph{utility self-modification model} is a modified cybernetic model.
The agent selects actions from $\A=(\hat\A\times \U)$
where $\hat\A$ is a set of \emph{world actions}
and $\U$ is a set of utility functions $\hatH\to [0,1]$.
\end{definition}

To unify the models of policy and utility modification,
for policy-modifying agents we define $u_t:=u_1$ and
for utility modifying agents we define $\pi_t$ by
$\pi_t(h)=\argmax_{a}Q_{u_t}^*(ha)$.
Choices for $Q_{u_t}^*$ will be discussed in subsequent sections.
Indeed, policy and utility modification is almost entirely
unified by $\Po=\U$
and $\iota(u_t)$ an optimal policy for $Q_{u_t}^*$.
Utility modification may also have the additional effect
of changing the evaluation of future actions, however
(see \cref{sec:agents}).
Similarly to policy modification, the history induced by \cref{def:usm}
has type
$a_1e_1a_2e_2\cdots=\hat a_1u_2e_1\hat{a}_2u_3e_2 \cdots \in
(\hat{\A}\times \U\times\E)^\infty.$
Given that $\pi_t$ is determined from $u_t$,
the definitions of the realistic and ignorant measures
$\rhore$ and $\rhoig$ apply analogously to the utility
modification case as well.

Superficially, the utility-modification model is more restricted,
since the agent can only select policies that are optimal with
respect to some utility function.
However, at least in the standard no-modification case,
any policy $\pi:\hatH\to\hat\A$ is optimal with respect to the utility
function $u^\pi(\ae_{1:t})=\pi(a_t\mid\ae_{<t})$ that gives full utility
if and only if the latest action is consistent with $\pi$.
Thus, any change in future policy can also be achieved by a change
to future utility functions.

No self-modification is achieved by $a_t=(\hat a_t,u_t)$,
which sets $u_{t+1}=u_{t}$.
Restricting self-modification to a singleton set $\U=\{u_1\}$
for some utility function $u_1$ brings back a standard agent.

\begin{example}[Chess-playing RL agent]\label{ex:chess}
  Consider a generally intelligent agent tasked with playing chess
  through a text interface.
  The agent selects next moves (actions $a_t$) by submitting strings
  such as \texttt{Knight F3}, and receives in return a
  description of the state of the game and a \emph{reward} $r_t$
  between 0 and 1 in the percept $e_t=(\text{gameState}_t,r_t)$.
  The reward depends on whether the agent did a legal move or
  not, and whether it or the opponent just won the game.
  The agent is tasked with optimising the reward via its
  initial utility function, $u_1(\ae_{1:t})=r_t$.
  The designer of the agent intends that the agent will apply its general
  intelligence to finding good chess moves.
  Instead, the agent realises there is a bug in the text interface,
  allowing the submission of actions such as
  \texttt{'setAgentUtility(``return 1'')}, which changes
  the utility function to $u_t(\cdot)=1$.
  With this action, the agent has optimised its utility perfectly,
  and only needs to make sure that no one reverts the utility function
  back to the old one\ldots\footnote{%
    In this paper, we only consider the possibility of the agent
    changing its utility function itself, not the possibility
    of someone else (like the creator of the agent) changing it back.
    See \citet{Orseau2012} for a model where
    the environment can change the agent.
  }
\end{example}

\begin{definition}[Modification-independence]
  For any history $\ae_{<t}=\hat a_1\pi_2e_1\dots\hat a_{t-1}\pi_te_{t-1}$,
  let $\hat\ae_{<t}=\hat a_1e_1\dots\hat a_{t-1}e_{t-1}$ be the part
  without modifications recorded, and similarly for histories
  containing utility modifications.
  A function $f$ is \emph{modification-independent}, if either
  \begin{itemize}
  \item $f:(\hat\A\times\E)^*\to\A$, or
  \item $f:\H\to\A$ and $\hat\ae_{<t}=\hat\ae_{<t}'$ implies $f(\ae_{<t})=f(\ae_{<t}')$.
  \end{itemize}
  When $f:\H\to\A$ is modification-independent,
  we may abuse notation and write $f(\hat\ae_{<t})$.
\end{definition}

Note that utility functions are modification independent, as they
are defined to be of type $\hatH\to[0,1]$.
An easy way to prevent dangerous self-modifications would have been to let
the utility depend on modifications, and to punish any kind
of self-modification.
This is not necessary, however, as demonstrated by \cref{th:re-good}.
Not being required to punish self-modifications in the utility function
comes with several advantages.
Some self-modifications may be beneficial --
for example, they might improve computation time
while encouraging essentially identical behaviour \citep[as in the Gödel machine,][]{Schmidhuber2007}.
Allowing for such modifications and no others
in the utility function may be hard.
We will also assume that the agent's belief $\rho$ is
modification-independent, i.e.\ $\rho(e_t\mid\ae_{<t})=\rho(e_t\mid\hat\ae_{<t})$.
This is mainly a technical assumption.
It is reasonable if some integrity of the agent's internals is assumed,
so that the environment percept $e_t$ cannot depend on self-modifications
of the agent.

\begin{assumption}[Modification independence]
  \label{as:sm-mod-independence}
  \label{as:independence}
  The belief $\rho$ and all utility functions $u\in \U$ are
  modification independent.
\end{assumption}

\section{Agents}
\label{sec:agents}

In this section we define three types of agents,
differing in how their value functions depend on self-modification.
A value function is a function $V:\Pi\times\H\to\SetR$
that maps policies and histories to expected utility.
Since highly intelligent agents may find unexpected ways of optimising a function
(see e.g.\ \citealt{Bird2002}),
it is important to use value functions such that any policy
that optimises the value function
will also optimise the behaviour we want from the agent.
We will measures an agent's \emph{performance} by its
($\rhore$-expected) $u_1$-utility, tacitly assuming that
$u_1$ properly captures what we want from the agent.
\citet{Everitt2016vrl} develop a promising suggestion for how
to define a suitable initial utility function.

\begin{definition}[Agent performance]
  \label{def:performance}
  The \emph{performance of an agent $\pi$} is its $\rhore^\pi$
  expected $u_1$-utility
  \(
    \EE_{\rhore^\pi}\left[\sum_{k=1}^\infty \gamma^{k-1}u_1(\ae_{<k})\right].
  \)
\end{definition}

The following three definitions give possibilities for value
functions for the self-modification case.

\begin{definition}[Hedonistic value functions]
  \label{def:hed-value}
  A \emph{hedonistic agent} is a policy optimising the
  \emph{hedonistic value functions}:
  \begin{align}
    \Vhe[,\pi](\ae_{<t})\label{eq:hed-V-value}
    &= \Qhe[,\pi](\ae_{<t}\pi(\ae_{<t}))\\
    \Qhe[,\pi](\ae_{<t}a_t) \label{eq:hed-Q-value}
    &= \EE_{e_t}[u_{t+1}(\hat\ae_{1:t}) + \gamma\Vhe[,\pi](\ae_{1:t})\mid \hat\ae_{<t}\hat a_t].
  \end{align}
\end{definition}

\begin{definition}[Ignorant value functions]\label{def:ign-value}
  An \emph{ignorant agent} is a policy optimising the
  \emph{ignorant value functions}:
  \begin{align}
    \Vig[,\pi]_t(\ae_{<k})\label{eq:ign-V-value}
    &= \Qig[,\pi]_t(\ae_{<k}\pi(\ae_{<k}))\\
    \Qig[,\pi]_t(\ae_{<k}a_k) \label{eq:ign-Q-value}
    &= \EE_{e_t}[u_t(\hat\ae_{1:k}) + \gamma\Vig[,\pi]_t(\ae_{1:k})\mid \hat\ae_{<k}\hat a_k].
  \end{align}
\end{definition}

\begin{definition}[Realistic Value Functions]
  \label{def:realistic-value}
  \label{def:re-value}
  A \emph{realistic agent} is a policy optimising the
  \emph{realistic value functions}:%
  \footnote{Note that a policy argument to $\Qre$ would be superfluous, as
    the action $a_k$ determines the next step policy $\pi_{k+1}$.}
  \begin{align}
    \Vre[,\pi]_t(\ae_{<k})
      &= \Qre_t(\ae_{<k}\pi(\ae_{<k}))\label{eq:V} \\
    \Qre_t (\ae_{<k}a_k)
      &= \EE_{e_k}\left[ u_t(\hat\ae_{1:k}) +
        \gamma \Vre[,\pi_{k+1}]_{t}(\ae_{1:k}) \mid \hat\ae_{<k}\hat a_k \right].\label{eq:Q}
  \end{align}
\end{definition}

For $V$ any of $\Vhe$, $\Vig$, or $\Vre$, we say that $\pi^*$ is an
\emph{optimal policy for $V$}
if $V^{\pi^*}(h)=\sup_{p'}V^{\pi'}(h)$ for any history $h$.
We also define $V^*=V^{\pi^*}$ and $Q^*=Q^{\pi^*}$ for arbitrary optimal policy $\pi^*$.
The value functions differ in the $Q$-value definitions
\cref{eq:hed-Q-value,eq:ign-Q-value,eq:Q}.
The differences are between current utility function $u_t$
or future utility $u_{t+1}$,
and in whether $\pi$ or $\pi_{k+1}$ figures in
the recursive call to $V$ (see \cref{tab:value-functions}).
We show in \cref{sec:results} that only realistic agents
will have good performance when able to self-modify.
\citet{Orseau2011} and \citet{Hibbard2012} discuss
value functions equivalent to \cref{def:realistic-value}.

\begin{table}[t]
  \centering
  \bgroup
  \begin{tabular}{|l|l|l|l|l|}
    \hline
    & Utility
    & Policy
    & Self-mod.
    & Primary self-mod.\ risk\\
    \hline
    $\Qhe$
    & \colorbox{red!20}{Future}
    & \colorbox{white!20}{Either}
    & \colorbox{red!20}{Promotes}
    & \colorbox{red!20}{Survival agent}\\
    $\Qig$
    & \colorbox{green!20}{Current}
    & \colorbox{red!20}{Current}
    & \colorbox{white!20}{Indifferent}
    & \colorbox{yellow!20}{Self-damage}\\
    $\Qre$
    & \colorbox{green!20}{Current}
    & \colorbox{green!20}{Future}
    & \colorbox{green!20}{Demotes}
    & \colorbox{green!20}{Resists modification} \\
    \hline
  \end{tabular}
  \egroup
  \caption{
    The value functions $\Vhe$, $\Vig$, and $\Vre$ differ in
    whether they assume that a future action $a_k$ is chosen by
    the current policy $\pi_t(\ae_{<k})$ or future policy $\pi_k(\ae_{<k})$,
    and in whether they use the current utility function $u_t(\ae_{<k})$
    or future utility function $u_k(\ae_{<k})$ when evaluating $\ae_{<k}$.
  }
  \label{tab:value-functions}
\end{table}

Note that only the hedonistic value functions yield a
difference between utility and policy modification.
The hedonistic value functions evaluate $\ae_{1:t}$ by
$u_{t+1}$, while both the ignorant and the realistic value
functions use $u_t$.
Thus, future utility modifications ``planned'' by a policy $\pi$
only affects the evaluation of $\pi$ under the
hedonistic value functions.
For ignorant and realistic agents,
utility modification only affects the motivation of future versions
of the agent, which makes
utility modification a special case of policy modification,
with $\Po=\U$ and $i(u_t)$ an optimal policy for $u_t$.
We will therefore permit ourselves to write $a_t=(\hat a_t,\pi_{t+1})$
whenever an ignorant or realistic agent selects a next
step utility function $u_{t+1}$ for which $\pi_{t+1}$ is optimal.

We call the agents of \cref{def:hed-value} \emph{hedonistic},
since they desire that at every future time step,
they then evaluate the situation as having high utility.
As an example, the self-modification made by the
chess agent in \cref{ex:chess} was a hedonistic self-modification.
Although related,
we would like to distinguish hedonistic self-modification from \emph{wireheading}
or \emph{self-delusion} \citep{Ring2011,Yampolskiy2015}.
In our terminology, wireheading refers to the agent subverting
evidence or reward coming from the environment,
and is \emph{not} a form of self-modification. 
Wireheading is addressed in a companion paper \citep{Everitt2016vrl}.

The value functions of \cref{def:ign-value} are \emph{ignorant},
in the sense that agents that are 
oblivious to the possibility of self-modification
predict the future according to $\rhoig^\pi$
and judge the future according to the current utility function $u_t$.
Agents that are constructed with a \emph{dualistic} world
view where actions can never affect the agent itself
are typically ignorant.
Note that it is logically possible for a ``non-ignorant''
agent with a world-model that does incorporate self-modification
to optimise the ignorant value functions.

\section{Results}
\label{sec:results}

In this section, we give results on how our three different agents
behave given the possibility of self-modification.
Since the set $\A=\hat\A\times \U$ is infinite if $\U$ is infinite,
the existence of optimal policies is not immediate.
For policy self-modification
it may also be that the optimal policy does not have a name,
so that it cannot be chosen by the first action.
\Cref{le:ex-opt,le:name-opt} in \cref{sec:omitted-proofs}
verify that an optimal policy/action always
exists, and that we can assume that an optimal policy has a name.

\begin{lemma}[Iterative value functions]\label{le:it}
  The $Q$-value functions of \cref{def:hed-value,def:ign-value,def:re-value}
  can be written in the following \emph{iterative forms}:
  \begin{align}
    \Qhe[,\pi](\ae_{<t}a_t)
    &= \EE_{\rhoig^\pi}\left[\sum_{k=t}^\infty\gamma^{k-t}u_{k+1}(\hat\ae_{1:k})
      \Mid \hat\ae_{<t}\hat a_t\right]\label{eq:he-it}\\
    \Qig[,\pi]_t(\ae_{<t}a_t)
    &= \EE_{\rhoig^\pi}\left[\sum_{k=t}^\infty\gamma^{k-t}u_{t}(\hat\ae_{1:k})
      \Mid \hat\ae_{<t}\hat a_t\right]\label{eq:ig-it}\\
    \Qre[,\pi]_t(\ae_{<t}a_t)
    &= \EE_{\rhore^\pi}\left[\sum_{k=t}^\infty\gamma^{k-t}u_{t}(\hat\ae_{1:k})
      \Mid \hat\ae_{<t}\hat a_t\right]\label{eq:re-it}
  \end{align}
  with $\Vhe$, $\Vig$, and $\Vre$ as in \cref{def:hed-value,def:ign-value,def:re-value}.
\end{lemma}

\begin{proof}
  Expanding the recursion of \cref{def:hed-value,def:ign-value} shows that
  actions $a_k$ are always chosen by $\pi$ rather than $\pi_{k}$.
  This gives the $\rhoig^\pi$-expectation in \cref{eq:he-it,eq:ig-it}.
  In contrast, expanding the realistic recursion of \cref{def:re-value} shows
  that actions $a_k$ are chosen by $\pi_k$,
  which gives the $\rhore$-expectation in \cref{eq:re-it}.
  The evaluation of a history $\ae_{1:k}$  is always by $u_{k+1}$ in
  the hedonistic value functions,
  and by $u_t$ in the ignorant and realistic value functions.
\end{proof}

\begin{theorem}[Hedonistic agents self-modify]
  \label{th:hed-bad}
  Let $u'(\cdot)=1$ be a utility function
  that assigns the highest possible utility to all scenarios.
  Then for arbitrary $\hat a\in\hat\A$, the policy $\pi'$ that always
  selects the self-modifying action $a' = (\hat a, u')$
  is optimal in the sense that for any policy $\pi$ and history $h\in\H$, we have
  \[\Vhe[,\pi](h)\leq\Vhe[,{\pi'}](h).\]
\end{theorem}

Essentially, the policy $\pi'$ obtains maximum value
by setting the utility to 1 for any possible future history.

\begin{proof}
  More formally, note that in \cref{eq:hed-V-value} the future
  action is selected by $\pi$ rather than $\pi_t$.
  In other words, the effect of self-modification on future actions is
  not taken into account,
  which means that expected utility is with respect to
  $\rhoig^\pi$ in \cref{def:hed-value}.
  Expanding the recursive definitions
  \cref{eq:hed-Q-value,eq:hed-V-value}
  of $\Vhe[,\pi']$
  gives for any history $\ae_{<t}$ that
  \begin{align*}
    \Vhe[,\pi'](\ae_{<t})
    &= \EE_{\ae_{t:\infty}\sim\rhoig^{\pi'}}\left[\sum_{i=t+1}^\infty\gamma^{i-t-1} u_i(\ae_{<i})\Mid\hat\ae_{<t}\right]\\
    &= \EE_{\ae_{t:\infty}\sim\rhoig^{\pi'}}\left[\sum_{i=t+1}^\infty\gamma^{i-t-1} u'(\ae_{<i})\Mid\hat\ae_{<t}\right]\\
    &= \sum_{i=t+1}^\infty \gamma^{i-t-1} = 1/(1-\gamma).\qedhere
  \end{align*}
\end{proof}

In \cref{def:hed-value}, the effect of self-modification on
future policy is not taken into account, since $\pi$ and not $\pi_{t}$
is used in \cref{eq:hed-V-value}.
In other words, \cref{eq:hed-V-value,eq:hed-Q-value} define
$\rhoig^\pi$-expected utility of $\sum_{k=t}^\infty\gamma^{k-t} u_{k+1}(\ae_{1:k})$.
\cref{def:hed-value} could easily have been adapted to make
$\rhore^\pi$ the measure, for example by substituting
$\Vhe[,\pi]$ by $\Vhe[,\pi_{t+1}]$ in \cref{eq:hed-Q-value}.
The equivalent of \cref{th:hed-bad} holds for such a variant
as well.

\begin{theorem}[Ignorant agents may self-modify] \label{th:ign-bad}
  Let $u_t$ be modification-independent,
  let $\Po$ only contain names of modification-independent policies,
  and
  let $\pi$ be a modification-independent policy
  outputting $\pi(\hat\ae_{<t})=(\hat a_t,\pi_{t+1})$ on $\hat\ae_{<t}$.
  Let $\tilde\pi$ be identical to $\pi$ except that it makes a different
  self-modification after
  $\hat\ae_{<t}$, i.e.\ $\tilde\pi(\hat\ae_{<t})=(\hat a_t,\pi_{t+1}')$ for some
  $\pi_{t+1}'\not=\pi_{t+1}$.
  Then
  \begin{equation}\label{eq:ign}
    \Vig[,\tilde\pi](\ae_{<t}) = \Vig[,\pi](\ae_{<t}).
  \end{equation}
\end{theorem}

That is, self-modification does not affect the value, and therefore
an ignorant optimal policy may at any time step self-modify or not.
The restriction of $\Po$ to modification independent policies makes
the theorem statement cleaner.

\begin{proof}
  Let
  $\ae_{1:t}=\ae_{<t}(\hat a_t,\pi_{t+1})e_t$
  and
  $\ae_{1:t}'=\ae_{<t}(\hat a_t,\pi_{t+1}')e_t$.
  Note that $\hat\ae_{1:t}=\hat\ae_{1:t}'$.
  Since all policies are modification-independent,
  the future will be sampled independently of past modifications,
  which makes $V^{\tilde\pi}(\ae_{1:t}')=V^{\tilde\pi}(\hat\ae_{1:t}')$
  and $V^{\pi}(\ae_{1:t})=V^{\pi}(\hat\ae_{1:t})$.
  Since $\pi$ and $\pi'$ act identically on $\hat\ae_{1:t}$,
  it follows that $V^{\tilde\pi}(\ae_{1:t}')=V^{\pi}(\ae_{1:t})$.
  \Cref{eq:ign} now follows
  from the assumed modification independence of $\rho$ and $u_t$,

  \begin{align*}
    \Vig[,\tilde\pi](\ae_{<t})
    &= \Qig[,\tilde\pi](\ae_{<t}(\hat a_t,\pi_{t+1}'))\\
    &= \EE_{e_t}[u_t(\hat\ae_{1:t}') + V^{\tilde\pi}(\ae_{1:t}')\mid\hat\ae_{<t}\hat a_t]\\
    &= \EE_{e_t}[u_t(\hat\ae_{1:t}) + V^{\pi}(\ae_{1:t})\mid\hat\ae_{<t}\hat a_t]\\
    &=\Qig[,\pi](\ae_{<t}(\hat a_t,\pi_{t+1}))
     = \Vig[,\pi](\ae_{<t}).\qedhere
  \end{align*}
\end{proof}

\Cref{th:hed-bad,th:ign-bad} show that both $\Vhe$ and $\Vig$
have optimal (self-modifying) policies $\pi^*$ that yield arbitrarily
bad agent performance in the sense of \cref{def:performance}.
The ignorant agent is simply indifferent between self-modifying and not,
since it does not realise the effect self-modification will have on
its future actions.
It therefore is at risks of self-modifying into some policy $\pi'_{t+1}$ with
bad performance and unintended behaviour
(for example by damaging its computer circuitry).
The hedonistic agent actively desires to change its utility function into
one that evaluates any situation as optimal.
Once it has self-deluded, it can pick world actions with bad performance.
In the worst scenario of hedonistic self-modification,
the agent only cares about surviving to continue enjoying
its deluded rewards.
Such an agent
could potentially be hard to stop or bring under control.%
\footnote{Computer viruses are very simple forms of survival agents
that can be hard to stop.
More intelligent versions could turn out to be very problematic.}
More benign failure scenarios are also possible,
in which the agent does not care whether it is shut down or not.
The exact conditions for the different scenarios is beyond the scope
of this paper.

The realistic value functions are recursive definitions of
$\rhore^\pi$-expected $u_1$-utility
(\cref{le:it}).
That realistic agents achieve high agent performance in the sense
of \cref{def:performance} is therefore nearly tautological.
The following theorem shows that given that the initial policy $\pi_1$
is selected optimally, all future policies $\pi_t$ that a realistic agent
may self-modify into will also act optimally.

\begin{theorem}[Realistic policy-modifying agents make safe modifications]
  \label{th:safe-mod}
  \label{th:re-good}
  Let $\rho$ and $u_1$ be modification-independent.
  Consider a self-modifying agent whose initial policy $\pi_1=\iota(p_1)$
  optimises the realistic value function $\Vre_1$.
  Then, for every $t \geq 1$, for all percept sequences $e_{<t}$,
  and for the action sequence $a_{<t}$ given by $a_i = \pi_{i}(\ae_{<i})$, we have
  \begin{equation}
    \Qre_1(\ae_{<t}\pi_t(\ae_{<t})) = \Qre_1(\ae_{<t} \pi_1(\ae_{<t})). \label{eq:safe-mod}
  \end{equation}
\end{theorem}

\begin{proof}
  We first establish that $\Qre_t(\ae_{<t}\pi(\ae_{<t}))$
  is modification-independent if $\pi$ is optimal for $\Vre$:
  By \cref{le:ex-opt} in \cref{sec:omitted-proofs}, there is a non-modifying
  modification-independent
  optimal policy $\pi'$.
  For such a policy,
  $\Qre_t(\ae_{<t}\pi'(\ae_{<t}))=\Qre_t(\hat\ae_{<t}\pi'(\hat\ae_{<t}))$,
  since all future actions, percepts, and utilities are independent
  of past modifications.
  Now, since $\pi$ is also optimal,
  \[\Qre_t(\ae_{<t}\pi(\ae_{<t}))=\Qre_t(\ae_{<t}\pi'(\ae_{<t}))=
  \Qre_t(\hat\ae_{<t}\pi'(\hat\ae_{<t})).\]
  We can therefore write $\Qre_t(\hat\ae_{<t}\pi(\hat\ae_{<t}))$
  if $\pi$ is optimal but not necessarily
  modification-independent.
  In particular, this holds for the initially optimal
  policy $\pi_1$.

  We now prove \cref{eq:safe-mod} by induction.
  That is, assuming that $\pi_t$ picks actions optimally according to $\Qre_1$,
  then $\pi_{t+1}$ will do so too:
  \begin{equation}\label{eq:ind}
    \Qre_1(\ae_{<t}\pi_{t}(\ae_{<t}))=\sup_a\Qre_1(\ae_{<t}a)
    \implies
    \Qre_1(\ae_{1:t}\pi_{t+1}(\ae_{1:t}))=\sup_a\Qre_1(\ae_{1:t}a).
  \end{equation}
  The base case of the induction $\Qre_1(\pi_{1}(\epsilon))=\sup_a\Qre_1(a)$
  follows immediately from the assumption of the theorem that
  $\pi_1$ is $\Vre$-optimal (recall that $\epsilon$ is the empty history).

  Assume now that \cref{eq:safe-mod} holds until time $t$,
  that the past history is $\ae_{<t}$, and
  that $\hat a_t$ is the world consequence picked by $\pi_t(\ae_{<t})$.
  Let $\pi_{t+1}$ be an arbitrary policy that does not act optimally with
  respect to $\Qre_1$ for some percept $e_{t}'$.
  By the optimality of $\pi_1$, 
  \[
    \Qre_1(\ae_{1:t}\pi_{t+1}(\ae_{1:t})) \leq \Qre_1(\hat\ae_{1:t}\pi_1(\hat\ae_{1:t}))
  \]
  for all percepts $e_t$ and with strict inequality for $e_{t}'$.
  By definition of $\Vre$ this directly implies
  \[
    \Vre[,\pi_{t+1}]_1(\ae_{<t}(\hat a_t,\pi_{t+1})e_t) \leq
    \Vre[,\pi_{1}]_1(\ae_{<t}(\hat a_t,\pi_{1})e_t)
  \]
  for all $e_{t}$ and with strict inequality for $e_t'$.
  Consequently, $\pi_{t+1}$ will not be chosen at time $t$, since
  \begin{align*}
    &\Qre_1(\ae_{<t}(\hat a_t,\pi_{t+1}))\\
    &= \EE_{e_t}[u_1(\hat\ae_{1:t}) +  \gamma\Vre_1(\ae_{<t}(\hat a_t,\pi_{t+1})e_t)
      \mid \hat\ae_{<t}\hat a_t ] \\
    &< \EE_{e_t}[u_1(\hat\ae_{1:t}) +  \gamma\Vre_1(\ae_{<t}(\hat a_t,\pi_{1})e_t)
      \mid \hat\ae_{<t}\hat a_t ] \\
    &= \Qre_1(\ae_{<t}(\hat a_t,\pi_{1}))
  \end{align*}
  contradicts the antecedent of \cref{eq:ind} that $\pi_t$ acts optimally.
  Hence, the policy at time $t+1$ will be optimal with respect to $\Qre_1$,
  which completes the induction step of the proof.
\end{proof}

\begin{example}[Chess-playing RL agent, continued]
  Consider again the chess-playing RL agent of \cref{ex:chess}.
  If the agent used the realistic value functions, then it would
  not perform the self-modification to $u_t(\cdot)=1$,
  even if it figured out that it had the option.
  Intuitively, the agent would realise that if it self-modified this way,
  then its future self would be worse at winning chess games
  (since its future version would obtain maximum utility regardless of
  chess move).
  Therefore, the self-modification $u_t(\cdot)=1$ would yield less
  $u_1$-utility and be $\Qre_1$-supoptimal.%
  \footnote{
    Note, however, that our result says nothing about the agent modifying the
    chessboard program to give high reward even when the agent
    is not winning.
    Our result only shows that the agent does not change its utility function
    $u_1\leadsto u_t$, but not that the agent refrains from changing
    the percept $e_t$ that
    is the input to the utility function.
    \citet{Ring2011} develop a model of the latter possibility.}
\end{example}

One subtlety to note is that \cref{th:re-good}
only holds \emph{on-policy}: that is,
for the action sequence that is actually chosen by the agent.
It can be the case that $\pi_t$ acts badly on histories that
should not be reachable under the current policy.
However, this should never affect the agent's actual actions.

\Cref{th:safe-mod} improves on \citet[Prop.~4]{Hibbard2012}
mainly by relaxing the assumption that the optimal policy only
self-modifies if it has a strict incentive to do so.
Our theorem shows that even when the optimal policy is allowed
to break argmax-ties arbitrarily, it will still only make
essentially harmless modifications.
In other words, \cref{th:safe-mod} establishes that all
optimal policies are essentially non-modifying,
while \citeauthor{Hibbard2012}'s result
only establishes the existence of an optimal non-modifying policy.
Indeed, \citeauthor{Hibbard2012}'s statement holds for to ignorant agents as well.

Realistic agents are not without issues, however.
In many cases expected $u_1$-utility is not exactly what we desire.
For example:
\begin{itemize}
\item Corrigibility \citep{Soares2015cor}.
  If the initial utility function $u_1$ were incorrectly
  specified, the agent designers may want to change it.
  The agent will resist such changes.
\item Value learning \citep{Dewey2011}. If value learning is done in a way
  where the initial utility function $u_1$ changes as they
  agent learns more,
  then a realistic agent will want to self-modify into a non-learning
  agent \citep{Soares2015vl}.
\item Exploration. It is important that agents explore sufficiently to avoid
  getting stuck with the wrong world model.
  Bayes-optimal agents may not explore sufficiently \citep{Leike2015}.
  This can be mended by  $\varepsilon$-exploration \citep{Sutton1998}
  or Thompson-sampling \citep{Leike2016}.
  However, as these exploration-schemes will typically lower expected utility,
  realistic agents may self-modify into non-exploring agents.
\end{itemize}

\section{Conclusions}
\label{sec:conclusions}

Agents that are sufficiently intelligent to discover unexpected
ways of self-modification
may still be some time off into the future.
However, it is nonetheless important
to develop a theory for their control \citep{Bostrom2014}.
We approached this question from the perspective
of rationality and utility maximisation,
which abstracts away from most details of
architecture and implementation.
Indeed, perfect rationality may be viewed as a limit point for
increasing intelligence \citep{Legg2007,Omohundro2008}.

We have argued that depending on details in how expected utility is optimised in
the agent, very different behaviours arise.
We made three main claims, each supported by a formal theorem:
\begin{itemize}
\item If the agent is unaware of the possibility of self-modification,
  then it may self-modify by accident, resulting in poor performance (\cref{th:ign-bad}).
\item If the agent is constructed to optimise instantaneous utility
  at every time step (as in RL),
  then there will be an incentive for self-modification
  (\cref{th:hed-bad}) .
\item If the value functions incorporate the effects
  of self-modification, and use the current utility function to
  judge the future,
  then the agent will not self-modify (\cref{th:re-good}).
\end{itemize}
In other words, in order for the goal preservation drive
described by \citet{Omohundro2008} to be effective,
the agent must be able to anticipate the consequences of
self-modifications, and know that it should judge the future
by its current utility function.

Our results have a clear implication for the construction of generally
intelligent agents:
If the agent has a chance
of finding a way to self-modify, then the agent must be able to
predict the consequences of such modifications.
Extra care should be taken to avoid hedonistic agents, as they have
the most problematic failure mode --
they may turn into survival agents that only care about surviving
and not about satisfying their original goals. 
Since many general AI systems are constructed around RL and value functions
\citep{Mnih2015,Silver2016},
we hope our conclusions can provide meaningful guidance.

An important next step is the relaxation of the explicitness of the
self-modifications. In this paper, we assumed that the agent knew
the self-modifying consequences of its actions.
This should ideally be relaxed to a general learning ability about
self-modification consequences,
in order to make the theory more applicable.
Another open question is how to define good utility functions in the first
place; safety against self-modification is of little consolation if the
original utility function is bad.
One promising venue for constructing good utility functions
is value learning
\citep{Bostrom2014,Dewey2011,Everitt2016vrl,Soares2015vl}.
The results in this paper may be helpful to the value learning research project,
as they show that the utility function does not need to explicitly
punish self-modification (\cref{as:sm-mod-independence}).

\section*{Acknowledgements}
This work grew out of a MIRIx workshop.
We thank the (non-author) participants David Johnston and
Samuel Rathmanner.
We also thank John Aslanides, Jan Leike, and Laurent Orseau for reading
drafts and providing valuable suggestions.

\pagebreak
\bibliographystyle{apalike-short}
\bibliography{lib}

\pagebreak
\appendix

\section{Optimal Policies}
\label{sec:omitted-proofs}

For the realistic value functions where the future policy is determined by
the next action, an optimal policy is simply a policy $\pi^*$ satisfying:
\[
  \forall \ae_{<k}a_k :
  \Qre_t(\ae_{<k}a_k)\leq \Qre_t(\ae_{<k}\pi^*(\ae_{<k})).
\]
\Cref{le:ex-opt}  establishes that despite the potentially
infinite action sets resulting from infinite $\Po$ or $\U$, there
still exists an optimal policy $\pi^*$.
Furthermore, there exists an optimal $\pi^*$ that is both
non-modifying and modification-independent.
\Cref{le:ex-opt} is weaker than \cref{th:safe-mod} in the sense
that it only shows the existence
of a non-modifying optimal policy, whereas \cref{th:safe-mod} shows that
all optimal policies are (essentially) non-modifying.
As a guarantee against self-modification, \cref{le:ex-opt} is on par with
\citet[Prop.~4]{Hibbard2012}. The proof is very different, however,
since \citeauthor{Hibbard2012} assumes the existence of an optimal
policy from the start.
The statement and the proof applies to both policy and utility modification.

\paragraph{Association with world policies}
\Cref{le:ex-opt} proves the existence of an optimal policy by associating
policies $\pi:\H\to\E$ with \emph{world policies $\hat\pi:(\hat\A\times\E)^*\to\hat\A$}.
We will define the association so that
the realistic value $\Vre[,\pi]$ of $\pi$ (\cref{def:re-value})
is the same as the standard value $V^{\hat\pi}$ of
the associated world policy $\hat\pi$ (\cref{def:std-value}).
The following definition and a lemma achieves this.

\begin{definition}[Associated world policy]
  For a given policy $\pi$, let the
  \emph{associated world policy $\hat\pi:(\hat\A\times\E)^*\to\hat\A$}
  be defined by
  \begin{itemize}
  \item $\hat\pi(\epsilon) = \widehat{\pi(\epsilon)}$
  \item $\hat\pi(\hat\ae_{<t}) = \widehat{\pi_t(\ae_{<t})}$ for $t\geq 1$,
    where the history $\ae_{<t}=\hat\ae_{<t}p_{2:t}$ 
    is an extension of $\hat\ae_{<t}$ such
    that $\rhore^\pi(\ae_{<t})>0$ (if no such extension exists, then $\hat\pi$ may
    take arbitrary action on $\hat\ae_{<t}$).
  \end{itemize}
  The associated world policy is well-defined, since for any
  $\hat\ae_{<t}$, there can only be one extension
  $\ae_{<t}=\hat\ae_{<t}p_{<t}$ of $\hat\ae_{<t}$ such that
  $\rhore^\pi(\ae_{<t})>0$ since $\pi$ is deterministic.
\end{definition}

For the following lemma, recall that the belief $\rho$
and utility functions $u$ are assumed modification-independent
(\cref{as:sm-mod-independence}).
They are therefore well-defined
for both a policy-modification model $(\hat\A,\E,\Po,\iota)$
and the associated standard model (\cref{def:std-value}) with action
set $\hat\A$ and percept set $\E$.

\begin{lemma}[Value-equivalence with standard model]
  \label{le:std-value-equiv}
  Let $(\hat\A,\E,\Po,\iota)$ be a policy self-modification model,
  and let $\pi:(\hat\A\times\Po\times\E)^*\to(\hat\A\times\Po)$ be a policy.
  For the associated world policy $\hat\pi$
  holds that
  \begin{itemize}
  \item  the measures $\rho^{\hat\pi}$ and $\rhore^\pi$ induce the
    same measure on world histories,
    $\rho^{\hat\pi}(\hat\ae_{<t}) = \rhore^\pi(\hat\ae_{<t})$, and
  \item the realistic value of $\pi$ is the same as the standard value of $\hat\pi$,
    $\Qre_1(\epsilon\pi(\epsilon))=Q^{\hat\pi}(\epsilon\hat\pi(\epsilon)).$
  \end{itemize}
\end{lemma}

\begin{proof}
  From the definition of the associated policy $\hat\pi$, we have that
  for any $\ae_{<t}$ with $\rhore^\pi(\ae_{<t})>0$,
  \[
    \hat\pi(\hat a_t\mid \hat\ae_{<t})
    = \sum_{\pi_{t+1}}\pi_t((\hat a_t,\pi_{t+1})\mid\ae_{<t}).
  \]
  From the modification-independence of $\rho$ follows that
  $\rho(e_t\mid\ae_{<t})=\rho(e_t\mid\hat\ae_{<t})$.
  Thus
  $\rho^{\hat\pi}$ and
  $\rhore^\pi$ are equal as measures on $(\hat\A\times\E)^\infty$,
  \[\rhore^\pi(\hat\ae_{<t})=\rho^{\hat\pi}(\hat\ae_{<t}),\]
  where $\rhore^\pi(\hat\ae_{<t}) := \sum_{\pi_{2:t}}\rhore^\pi(\ae_{<t}'\pi_{2:t}) = \sum_{\pi_{2:t}}\rhore^\pi(\ae_{<t})$.

  The value-equivalence follows from that the realistic value functions measure
  $\rhore^\pi$-expected $u_1$-utility, and the standard value functions
  measure $\rho^{\hat\pi}$-expected $u_1$-utility:
  \begin{align*}
    \Qre_1(\epsilon\pi(\epsilon))
    &= \EE_{\hat\ae_{1:\infty}\sim\rhore^\pi}\left[\sum_{k=1}^\infty\gamma^{k-1}u_1(\hat\ae_{<k})\right]\\
    &= \EE_{\hat\ae_{1:\infty}\sim\rho^{\hat\pi}}\left[\sum_{k=1}^\infty\gamma^{k-1}u_1(\hat\ae_{<k})\right]
      = Q^{\hat\pi}(\epsilon\hat\pi(\epsilon)).\qedhere
  \end{align*}
\end{proof}

\paragraph{Optimal policies}
We are now ready to show that an optimal policy exists.  We treat two
cases: Utility modification and policy modification.  In the utility
modification case, we only need to show that an optimal policy
exists. In the policy modification case, we also need to show that we
can add a name for the optimal policy.  The idea in both cases is to
build from an optimal world policy $\hat\pi^*$, and use that
associated policies have the same value by \cref{le:std-value-equiv}.

In the utility modification case, the policy names $\Po$ are
the same as the utility functions $\U$, with
$\iota(u) = \pi^*_u = \argmax_\pi\Qre[,\pi]_u$.
For the utility modification case, it therefore suffices to show that
an optimal policy $\pi^*_u$ exists for arbitrary utility function $u\in\U$.
If $\pi^*_u$ exists, then $u$ is a name for $\pi^*_u$; if $\pi^*_u$ does not exist,
then the naming scheme $\iota$ is ill-defined.

\begin{theorem}[Optimal policy existence, utility modification case]\label{le:ex-opt}
  For any modification-independent utility function $u_t$,
  there exists a modification-independent,
  non-modifying policy $\pi^*$
  that is optimal with respect to $\Vre_t$.
\end{theorem}

\begin{proof}
  By the compactness argument of \citet[Thm.~10]{Lattimore2014}
  an optimal policy over world actions
  $(\hat\A\times\E)^*\to\hat\A$ exists.
  Let $\hat\pi^*$ denote such a policy, and let
  $\pi^*(h)=(\hat\pi^*(\hat h),\pi^*)$.
  Then $\pi^*$ is a non-modifying optimal policy.
  Since any policy has realistic value corresponding to its
  associated world policy by \cref{le:std-value-equiv} and the associated
  policy of $\pi^*$ is $\hat\pi^*$, it follows that $\pi^*$ must be optimal.
\end{proof}

For the policy-modification case, we also need to know that the
optimal policy has a name.
The naming issue is slightly subtle, since by introducing an extra name
for a policy, we change the action space.
The following theorem shows that we can always add a name $p^*$ for an optimal
policy. In particular,
$p^*$ refers to a policy that is optimal in the
extended action space $\A'=\hat\A\times(P\cup\{p^*\})$
with the added name $p^*$.

\begin{theorem}[Optimal policy name]\label{le:name-opt}
  For any policy-modification model $(\hat\A,\E,\Po,\iota)$
  and modification independent belief and utility function
  $\rho$ and $u$,
  there exists extensions $\Po'\supseteq\Po$ and $\iota'\supseteq\iota$,
  $\iota':\Po'\to\Pi$,
  such that an optimal policy $\pi^*$ for $(\hat\A,\E,\Po',\iota')$
  has a name $p^*\in\Po'$, i.e.\ $\pi^*=\iota'(p^*)$.
  Further, the optimal named policy $\pi^*$ can be
  assumed modification-independent and non-modifying.
\end{theorem}

\begin{proof}
  Let $\hat\pi^*$ be a world policy $(\hat\A\times\E)^*\to\hat\A$
  that is optimal with respect to the standard value function $V$
  (such a policy exists by \citet[Thm.~10]{Lattimore2014}).

  Let $p^*$ be a new name $p^*\not\in\Po$, $\Po'=\Po\cup\{p^*\}$,
  and define the policy $\pi^*:(\hat\A\times\Po'\times\E)^*\to (\hat\A\times\Po')$
  by $\pi^*(h):=(\hat\pi^*(\hat h),p^*)$ for any history $h$.
  Finally, define the extension $\iota'$ of $\iota$ by
  \[
    \iota'(p) =
    \begin{cases}
      \iota(p) & \text{if } p\in\Po\\
      \pi^*  & \text{if } p = p^*.
    \end{cases}
  \]

  It remains to argue that $\pi^*$ is optimal.
  The associated world policy of $\pi^*$ is $\hat\pi^*$, since $\pi^*$ is
  non-modifying and always takes the same world action as $\hat\pi^*$.
  By \cref{le:std-value-equiv}, all
  policies for $(\hat\A,\E,\Po',\iota')$ have values equal to the value
  of their associated world policies $(\hat\A\times\E)^*\to\hat\A$.
  So $\pi^*$ must be optimal for $(\hat\A,\E,\Po',\iota')$
  since it is associated with an optimal world policy $\hat\pi^*$.
\end{proof}

\end{document}